\documentclass{article}

\usepackage{microtype}
\usepackage{graphicx}
\usepackage{subcaption}
\usepackage{booktabs} 

\usepackage{hyperref}

\usepackage[preprint]{icml2026}

\usepackage{amsmath}
\usepackage{amssymb}
\usepackage{mathtools}
\usepackage{amsthm}

\usepackage{caption}
\usepackage{subcaption}
\usepackage{graphicx}
\usepackage{url}
\usepackage{appendix}
\usepackage{chngcntr}
\usepackage{etoolbox}
\usepackage{lipsum}
\usepackage{subcaption}
\usepackage{array}
\usepackage{float}
\usepackage{amsfonts}
\usepackage{amsthm}
\usepackage{amsmath}
\usepackage{multicol}
\usepackage{multirow}

\usepackage{color}

\def\RR{\mathbb{R}}
\def\EE{\mathbb{E}}

\def\OO{\mathcal{O}}
\def\BB{\mathcal{B}}
\def\FF{\mathcal{F}}

\usepackage[capitalize,noabbrev]{cleveref}

\theoremstyle{plain}
\newtheorem{theorem}{Theorem}[section]
\newtheorem{proposition}[theorem]{Proposition}
\newtheorem{lemma}[theorem]{Lemma}

\theoremstyle{definition}

\newtheorem{assumption}[theorem]{Assumption}
\theoremstyle{remark}

\usepackage[textsize=tiny]{todonotes}

\icmltitlerunning{A Trainable Optimizer}

\begin{document}

\twocolumn[
  \icmltitle{A Trainable Optimizer}



  \icmlsetsymbol{equal}{*}

  \begin{icmlauthorlist}
    \icmlauthor{Ruiqi Wang}{1}
    \icmlauthor{Diego Klabjan}{1}
  \end{icmlauthorlist}

  \icmlaffiliation{1}{Industrial Engineering and Management Sciences, Northwestern University, Evanston, IL, USA, 60201}

  \icmlcorrespondingauthor{Ruiqi Wang}{RuiqiWang2025@u.northwestern.edu}

  \vskip 0.3in
]



\printAffiliationsAndNotice{}  

\begin{abstract}
  The concept of learning to optimize involves utilizing a trainable optimization strategy rather than relying on manually defined full gradient estimations such as ADAM. We present a framework that jointly trains the full gradient estimator and the trainable weights of the model. Specifically, we prove that pseudo-linear TO (Trainable Optimizer), a linear approximation of the full gradient, matches SGD’s convergence rate while effectively reducing variance. Pseudo-linear TO incurs negligible computational overhead, requiring only minimal additional tensor multiplications. To further improve computational efficiency, we introduce two simplified variants of Pseudo-linear TO. Experiments demonstrate that TO methods converge faster than benchmark algorithms (e.g., ADAM) in both strongly convex and non-convex settings, and fine tuning of an LLM. 
\end{abstract}

\section{Introduction}

 Gradient-based stochastic optimization methods are fundamental to solving many machine learning problems. Given a set of loss functions $\{f_n(w)\}_{n=1}^N$ where $w \in \mathbb{R}^d$, the full loss is defined as $F(w)=\frac{1}{N} \sum_{i=1}^N f_n(w)$. For a subset $\BB\subset\{1,\ldots,N\}$, the mini-batch loss is given by $F^{\BB}(w) := \frac{1}{|\BB|}\sum_{n\in\BB} f_n(w)$. Let $w^*$ be the global minimum of $F(w)$. Popular approaches, such as Stochastic Gradient Descent (SGD), RMSProp \citep{hinton2012neural}, Adagrad \citep{duchi2011adaptive}, and ADAM \citep{kingma2014adam}, can be unified under a general framework. In such a framework, the update direction depends on past iterations $w_1,\ldots, w_t$, past stochastic gradients $g_1,\ldots, g_t$ and the optimizer's variables $\theta_t$, following the update rule
\begin{eqnarray}
    w_{t+1} = w_t - \gamma_t \widehat G_t,\nonumber
\end{eqnarray}
where $\gamma_t$ is the learning rate and $\widehat G_t$ denotes the update direction.

We list the examples of functions $\widehat G_t$ in Table~\ref{tab:egofG}.
\begin{table}[H]
    \centering
    \begin{tabular}{c|c}
      \hline Algorithm   & $\widehat G_t$ \\ \hline
       SGD  & $g_t$\\
       Momentum & $(1-\beta_1)\sum_{k=1}^t \beta_1^{t-k} g_k$ \\ 
       Adagrad & $\frac{g_t}{\sqrt{\sum_{k=1}^t g_k \odot g_k / t + \epsilon}}$\\
       RMSProp & $\frac{g_t}{\sqrt{(1-\beta)\sum_{k=1}^t \beta ^{t-k} g_k \odot g_k +\epsilon}}$\\
       ADAM & $\frac{(1-\beta_1)\sum_{j=1}^t \beta_1^{t-j} g_j }{\sqrt{(1-\beta_2)\sum_{k=1}^t \beta_2 ^{t-k} g_k\odot g_k +\epsilon}}$\\ \hline
    \end{tabular}
    \caption{Examples of $\widehat G_t$}
    \label{tab:egofG}
\end{table}

Learning to Optimize (L2O) automates the development of optimization strategies by replacing traditional, manually designed $\widehat G_t$ functions with learned counterparts \citep{chen2022learning}. The core premise of L2O is to employ machine learning models, typically neural networks, to discover task-specific optimization policies that outperform conventional handcrafted rules. Unlike classical optimizers with fixed update rules and hyperparameters, L2O models are meta-trained across diverse problems to generalize to unseen tasks, making them particularly effective for complex, high-dimensional spaces such as neural architecture search and scientific computing. L2O models contain an offline procedure, where the learnable optimizer is trained on a set of similar tasks, and an online procedure, where the trained optimizer optimizes a new unseen task from the same task distribution.

Generally, an L2O approach parameterizes the update direction $\widehat G_t$ as a function of past values $w_1,\ldots,w_t$ of trainable model weights and past gradients $g_1,\ldots, g_t$ with optimizer variable $\theta$ being fixed. For a given task, $\theta$ is updated at the end of the optimization procedure, based on the performance of the optimizer on that task. With a learned $\theta$, an unseen task is further handled with the following update rule for the model weights
\begin{eqnarray}
    w_{t+1} = w_t -\gamma_t \widehat G_t(w_1,\ldots,w_t; g_1,\ldots,g_t,\theta). \nonumber
\end{eqnarray}

Practically, for a given new task, similar tasks are not necessarily available, which makes the offline procedure of L2O not possible. We propose a departure from standard L2O methods—which train $\widehat G_t$ across broad problem distributions—by instead co-training the optimizer and weights of the model for a single task. Our new approach replaces the fixed multidimensional optimizer variable $\theta\in\RR^{\bar d}$ with $\theta_t$, which is updated along with $w_t$ simultaneously at each iteration. The trainable weights of the model are updated as
\begin{eqnarray}
\label{eqn:updatew}
    w_{t+1} = w_t - \gamma_t \widehat G_t(w_1,\ldots,w_t;g_1,\ldots,g_t;\theta_t).
\end{eqnarray}
It is easy to see that (\ref{eqn:updatew}) captures all optimizers listed in Table~\ref{tab:egofG}. Optimizer variable $\theta_t$ is updated based on the gradient of an approximation loss function $$l(\theta_t;w_1,\ldots,w_t;g_1,\ldots,g_t),$$ following the update rule:
\begin{eqnarray}
\label{eqn:updatetheta}
    \theta_{t+1} = \theta_t - \beta_t \nabla_{\theta_t} l(\theta_t;w_1,\ldots,w_t;g_1,\ldots,g_t).
\end{eqnarray}

Despite L2O's empirical success in accelerating convergence, its theoretical guarantees remain unstudied. We establish that for specific parameterizations of $\widehat G_t$ and setting of $l$ to be an $L_2$ error with the ground truth being the gradient of $F$,
\begin{eqnarray}
    &&l(\theta_t;w_1,\ldots,w_t; g_1\ldots,g_t) \nonumber\\
    &&=\frac{1}{2}\left\| g_t - \widehat G_t(w_1,\ldots,w_t;g_1,\ldots,g_t,\theta_t)\right\|_2^2,\nonumber
\end{eqnarray}
both the gradient approximation error $\left\|\widehat G_t-\nabla F(w_t)\right\|$ and the optimality gap $\|w_t - w^*\|$ converge to zero simultaneously.

The relationship between convergence and variance reduction in gradient-based methods has been well-studied. \citep{johnson2013accelerating} proposed SVRG, achieving variance reduction through snapshot gradients and bias correction via full-batch gradients, guaranteeing exponential convergence for strongly convex losses.  \citep{zaheer2018adaptive} proved ADAM's convergence under growing mini-batch sizes, while \citep{qian2020impact} and \citep{Wang2022Divergence} further established a crucial link between variance reduction and convergence. A fundamental criterion for analyzing stochastic optimization methods involves verifying whether the estimation variance vanishes as $t\to\infty$. Our framework inherently satisfies this property, constituting a trainable variance reduction algorithm.

In this work, we focus on pseudo-linear approximations, where $\widehat G_t$ is formulated as a linear function with regards to $w_t$, i.e., $\widehat G_t = A_t w_t + b_t$ with optimizer variables $A_t\in\RR^{d\times d}$, $b_t\in \RR^{d}$, and $\theta_t = (A_t, b_t)$. We call it pseudo since $A_t$ and $b_t$ are functions of $w_1,\ldots,w_{t-1}$ and $g_1,\ldots,g_{t-1}$. This pseudo-linear parameterization is motivated by two key considerations. First, it draws inspiration from Taylor expansion principles where lower-order terms often provide effective approximations. Second, it enables efficient computation, as the gradient of the approximation loss can be calculated explicitly by simple tensor multiplications, which requires negligible additional computational time.

Our theoretical analysis establishes that for strongly convex loss functions, the parameters $w_t$ under the pseudo-linearly approximated gradients converge to the optimal point $w^*$ at the rate $\EE\left[\left\|w_t - w^*\right\|_2^2\right] \leq \OO(1/t)$. Furthermore, we show that the approximation variance $\EE\left[\left\|\widehat G_t - \nabla F(w_t)\right\|_2^2\right]$ also converges to zero at the same $\OO(1/t)$ rate. Comparing with the variance of SGD and ADAM, which does not converge to zero as $t\to\infty$, these results confirm that our trained optimizer functions as an effective variance reduction algorithm, which we name Trainable Optimizer (TO).

To address computational concerns, we propose two simplified variants of pseudo-linear TO:
\begin{itemize}
    \item Diagonal TO restricts $A_t$ to diagonal matrices,
    \item RankOne TO uses rank-one matrices for $A_t$.
\end{itemize}
The two variants significantly reduce the number of variables in the optimizer, which are more suitable with the practical cases with limited memory. Notably, we prove that momentum-based SGD emerges as a special case of pseudo-linear TO when appropriate learning rates and initializations $A_0$ and $b_0$ are selected.

Our comprehensive experimental evaluation compares three TO variants (Pseudo-linear (Full-TO), Diagonal, RankOne) across diverse machine learning tasks spanning three categories:
\begin{itemize}
    \item Strongly Convex losses (e.g., $L_2$-regularized logistic regression),
    \item Convex but not strongly convex losses (e.g., logistic regression without regularization),
    \item Non-convex losses (e.g., ResNet classification, Llama fine-tuning).
\end{itemize}
The results demonstrate that TO significantly outperforms ADAM in both strongly convex and complex non-convex settings, while matching ADAM's performance on standard convex tasks.  

Our main contributions are as follows.
\begin{itemize}
    \item We propose a new framework that updates the optimizer variables simultaneously with trainable weights of the model, based on (\ref{eqn:updatew}) and (\ref{eqn:updatetheta}).
    \item We show the $\OO(1/t)$ convergence of Pseudo-linear TO for strongly convex losses. 
    \item We propose two variants of pseudo-linear TO: Diagonal TO and RankOne TO. They reduce the memory requirement of pseudo-linear TO and are better suitable for large models.
    \item We show by means of experiments that TO outperforms ADAM on strongly convex problems and complex non-convex problems.
\end{itemize}

In Section 2, we review related work on trainable optimizers and variance-reduction optimization methods. Section 3 presents the TO framework with linear approximation and demonstrates the convergence of model parameters to the optimal point while reducing the approximation variance to zero. This section also introduces two simplified versions—Diagonal TO and RankOne TO. In Section 4, we present numerical experiments illustrating the improved convergence rates of the proposed methods.

\section{Related Work}

\subsection{Learning to Optimize}
The L2O paradigm represents a significant shift from traditional optimization methods, replacing theoretically-derived update rules with data-driven approaches that leverage machine learning. Pioneering work by \citep{andrychowicz2016learning} demonstrated that recurrent neural networks (RNNs) can learn effective optimization strategies that outperform conventional methods like SGD. Subsequent advances incorporated reinforcement learning (RL) frameworks \citep{li2017learning,bello2017neural}, enabling the discovery of specialized optimization algorithms for diverse applications ranging from neural architecture search to black-box optimization. Our work extends these foundations while introducing several key innovations. First, unlike standard L2O approaches that solely optimize the target model's parameters, our framework simultaneously optimizes both the model parameters and the optimization strategy itself. This co-optimization enables dynamic adaptation to problem-specific characteristics. Second, while most L2O methods lack theoretical guarantees, we provide rigorous convergence proofs showing simultaneous convergence of both trainable weights of the model and the gradient approximation error. By employing a carefully designed linear approximation, our approach achieves both computational efficiency and theoretical soundness.

\subsection{Variance Reduction}
Variance reduction has emerged as a fundamental technique for improving stability and the convergence rate of stochastic optimization methods. One particularly influential method is Stochastic Variance Reduced Gradient (SVRG), originally introduced by \citep{NIPS2013_ac1dd209}. SVRG operates by maintaining a snapshot model to effectively correct the noisy estimates obtained from stochastic gradients, thereby making the entire optimization process both faster and more numerically stable. Another important approach is SAGA, proposed by \citep{NIPS2014_ede7e2b6}, which builds upon SVRG's foundation while streamlining the process through an efficient history of gradients, making it more practical for large-scale implementations. Recent work includes the introduction of Variance Reduced Adam (VRADAM) by \citep{Wang2022Divergence}, which combines variance reduction techniques with ADAM. This hybrid approach successfully resolves ADAM's well-known divergence issues while carefully preserving its advantages in training speed and the convergence rate for challenging deep learning tasks. While other recent works, such as \citep{huang2022superadamfasteruniversalframework} and \citep{Cutkosky2019Momentum}, have demonstrated various approaches combining variance reduction techniques with modern optimization algorithms, our work makes distinct contributions by specifically focusing on utilizing the idea of trainable optimizers for variance reduction. We provide theoretical and empirical evidence showing that the gradient variance can be systematically reduced through the optimization process itself as the trainable optimizer variables are progressively refined during training.

\section{Trainable Optimizer with Pseudo-linear Approximation}
\subsection{The Pseudo-linear TO Algorithm}
In this section, we present our approach for approximating the full gradient with a trainable pseudo-linear function. Consider a stationary point $w^*$ of $F$, where we assume that $F$ is second-order differentiable at $w^*$. Drawing inspiration from the Taylor expansion of $\nabla F(w)$ centered at $w^*$, we have 
\begin{eqnarray}
    \nabla F(w) &\approx& \nabla F(w^*) + \nabla ^2F(w^*)(w-w^*)\nonumber\\
    &=& \nabla ^2F(w^*)(w-w^*),\nonumber
\end{eqnarray}
which yields a linear function of $w$. We formulate the approximated gradient as $\widehat G_t = A_t w_t+b_t$, where $A_t$ and $b_t$ represent the variables of the optimizer. The corresponding approximation loss is defined by
\begin{eqnarray}
l((A_t,b_t);w_t;g_t):=\frac{1}{2}\left\| g_t-A_tw_t-b_t\right\|_2^2,\nonumber
\end{eqnarray}
with gradients computed as
\begin{eqnarray}
\nabla_{A_t} l((A_t,b_t);w_t;g_t) &=& - (g_t - A_t w_t - b_t) w_t^\top    \nonumber\\
\nabla_{b_t} l((A_t,b_t);w_t;g_t) &=& - (g_t - A_t w_t - b_t).\nonumber
\end{eqnarray}

We introduced the Trainable Optimizer (TO) with pseudo-linear approximation in Algorithm~\ref{alg:linearTVR}. Notably, lines 5 and 6 of the algorithm implement updates to the parameters $A$ and $b$ through the gradient descent on the loss function $l((A_t,b_t);w_t;g_t)$.

The pseudo-linear TO framework admits an important interpretation as a generalized momentum-based SGD algorithm. This connection becomes exact under specific initialization conditions: when we set $A_0=0$, $b_0=0$, $\alpha_t=0$ and maintain a constant $\beta_t=\beta$ for some $0<\beta<1$, the update rules in Algorithm~\ref{alg:linearTVR} are simplified to
\begin{eqnarray}
    A_t = 0, b_t = \beta g_t + (1-\beta) b_{t-1}, \widehat G_t = b_t,\nonumber
\end{eqnarray}
which recovers precisely the standard momentum-based SGD formulation. The initialization of parameters $A_0$ and $b_0$ presents practical challenges. However, based on our theoretical analysis and empirical observations, we strongly recommend initializing both to zero, as this initialization strategy provides effective practical performance during the initial training phase.

\begin{algorithm}[tb]
\caption{Pseudo-linear TO (Full-TO)}
\label{alg:linearTVR}
\begin{algorithmic}[1]
\REQUIRE Learning rates $\alpha_t$, $\beta_t$, $\gamma_t$.
\STATE Initialize $w_1$, $A_0$ and $b_0$.
\FOR {$t=1,\ldots, T$}
\STATE Sample $\BB_t \subset \{1,\ldots,N\}$ such that $|\BB_t|=b$
\STATE $g_t\gets \nabla F^{\BB_t}(w_t)$
\STATE $A_t \gets A_{t-1} + \alpha_t(g_t -A_{t-1}w_t - b_{t-1}) w_t^{\top}$
\STATE $b_t \gets b_{t-1} + \beta_t(g_t - A_{t-1}w_t - b_{t-1})$
\STATE $\widehat G_t \gets A_t w_t + b_t$
\STATE $w_{t+1} \gets   w_t - \gamma_t \widehat G_t $
\ENDFOR
\end{algorithmic}
\end{algorithm}

\subsection{Analyses}
In this section we theoretically prove the convergence of Pseudo-linear TO under strongly convex loss. Theoretical analyses of gradient-based algorithms often rely on assumptions of bounded iterates and gradients. To this end, we modify Algorithm~\ref{alg:linearTVR} by changing line 8 to \begin{eqnarray}
    w_{t+1} \gets \Pi_\FF(w_t-\gamma_t\widehat G_t).\nonumber
\end{eqnarray}
The projection operator $\Pi_{\FF}(w)$, which maps vectors to the bounded feasible set $\FF$, is mathematically defined as 
\begin{eqnarray}
\Pi_\FF(w):=\arg\min_{w'\in\FF}\|w-w'\|_2.    \nonumber
\end{eqnarray}
Let $D_w=\sup_{w\in\FF}\|w\|_2<\infty$. In our context, we include additional optimizer variables $A_t$ and $b_t$. In this section, we first show that under our arrangement of the algorithm, the trajectory of gradients are uniformly bounded. Then, we show that $A_t$ and $b_t$ are also bounded. All proofs are in the appendix.

For the theoretical analysis, we make the following assumptions. We assume that $\FF$ is closed convex. This implies that $w_{t+1}$ is a unique minimizer. Let $\bar\FF$ be any open set that contains $\FF$. Since $\bar\FF$ is open, the derivatives on $\FF$ can be appropriately defined.

\begin{assumption}
\label{assmpt:F}
The full loss $F$ and minibatch losses $F^\BB$ for all $\BB$ satisfy the following conditions.
\begin{enumerate}
    \item Feasibility set $\FF$ must be such that $w^*\in\FF$, where $w^*$ is an optimal solution to the unconstrained problem.
    \item For all $\BB\subset\{1,\ldots,N\}$ with $|\BB|=b$, $\nabla F^\BB (w)$ is Lipschitz continuous on $\bar\FF$, i.e., there exists $L>0$ such that for all $w,w'\in\bar\FF$, we have $\left\|\nabla F^\BB(w) - \nabla F^\BB (w')\right\|\leq L\left\|w-w'\right\|_2$.
    \item  There exists $0<D_V<+\infty$ such that for all $w\in\bar\FF$, $\EE_\BB\left[\left\|\nabla F^{\BB}(w)-\nabla F(w)\right\|_2^2\right]\leq D_V,$ where $\BB$ is a random subset of $\{1,\ldots,N\}$ with $|\BB| = b$.
\end{enumerate}
\end{assumption}

The conditions are commonly made when analyzing convergence of gradient based optimization methods. Remarkably, we avoid making the uniform bounded gradient assumption, which contradicts strong convexity. However, given the Assumption and the bound for the trajectory $w_t$, i.e., $\|w_t\|_2\leq D_w$ for all $t$, the stochastic gradients $g_t$ in each iterate are also uniformly bounded for each $t$, as shown in the following proposition.
\begin{proposition}
\label{lem:Bddg}
    Given $\|w_t\|\leq D_w$, there exists $0<D_G<+\infty$ such that for all $t$, $\|g_t\|_2 \leq D_G$.
\end{proposition}

Furthermore, under specific requirements on learning rates, the optimizer variables $A_t$ and $b_t$ are also bounded. In what follows, all matrix norms are spectral.

\begin{proposition}
    \label{lem:BddAb}
    Let $S_\alpha := \sum_{s=1}^\infty \alpha_s < 1/D_w^2<\infty$ and let $\beta_t \leq 1$ for all $t$. There exist $D_A$ and $D_b$ independent on $t$, such that $\|A_t\|_2\leq D_A$ and $\|b_t\|_2 \leq D_b$ for all $t$.
\end{proposition}
The requirement $\sum_{t=1}^\infty \alpha_s < \infty$ can be simply satisfied by $\alpha_s=6/\pi^2D_w^2s^2$. We prove this proposition by properly applying triangle inequalities and the principle of induction.

Since $F(w)=\frac{1}{\binom{N}{b}}\sum_{\BB: |\BB|=b} F^\BB(w)$, Item 3 in Assumption naturally implies that 
\begin{eqnarray}
   && \|\nabla F(w)-\nabla F(w')\|_2 \nonumber\\
    &\leq& \frac{1}{\dbinom{N}{b}}\sum_{\BB:|\BB|=b} \|\nabla F^{\BB}(w)-\nabla F^{\BB}(w')\|_2 \nonumber\\
    &\leq& L\|w-w'\|_2,\nonumber
\end{eqnarray}
i.e., $\nabla F$ is also Lipschitz continuous. Therefore we have $\left\|\nabla F(w_t)\right\|_2 \leq L\left\|w_t-w^*\right\|\leq 2LD_w < D_G$ and we have $\max\left\{\left\|\nabla F(w_t)\right\|_2, \left\|\widehat G_t\right\|_2 \right\}\leq D_G$. Using the previously defined constants $L, D_A$ and $D_G$, we present the convergence result of Pseudo-linear TO for strongly convex losses.

\begin{theorem}
\label{thm:linTVR}
    Given Assumption and that $F$ is a $c$ strongly convex function, setting $\gamma_t = \gamma/(t+\mu)$ and $\beta_t = \beta/(t-1+\mu)$, where $\gamma$ and $\beta$ satisfy
\begin{eqnarray}
    \gamma > \frac{1}{c}\;\;\text{and}\;\;
    \beta > 1 + \frac{4\gamma^2 L^2\sqrt{D_A^2+L^2}}{ c(\gamma c- 1)} + 2\gamma \sqrt{D_A^2 +L^2} ,\nonumber
\end{eqnarray}
and $\alpha_t = \alpha /(t-1+\mu)^2$ with $\sum_{s=1}^\infty \alpha_s < 1/D_w^2$, and $\mu$ to be large enough such that $\alpha_1 < 1, \beta_1<1$, $\gamma_1 < 1$, and $\alpha D_w^2 /\mu + \beta < \mu$, there exist $M_1,M_2>0$ such that for all $t$, we have
\begin{eqnarray}
    \EE\left[\left\|w_t-w^*\right\|_2^2\right]&\leq& \frac{M_1}{t+\mu}\nonumber\\
    \EE\left[\left\|\widehat G_t-\nabla F(w_t)\right\|_2^2\right]&\leq& \frac{M_2}{t+\mu}.\nonumber
\end{eqnarray}
\end{theorem}
To prove the theorem, we jointly bound $\EE\left[\left\|w_t-w^*\right\|_2^2\right]$ and $\EE\left[\left\|\widehat G_t-\nabla F(w_t)\right\|_2^2\right]$ recurrently and apply induction. The expectations are over history up to $t$. Theorem~\ref{thm:linTVR} shows that for strongly convex losses, Pseudo-linear TO achieves an $\OO(1/t)$ convergence rate, matching the rates of SGD and Momentum. Comparing to ADAM, which has strongly convex examples with non-convergent variance \citep{Wang2022Divergence}, namely
$\EE\left[\|\widehat G_t^{\mathrm{(ADAM)}} -\nabla F(w_t)\|^2\right] \not\to 0$, the convergence result for Pseudo-linear TO automatically achieves the reduction of variance. 

Relaxing the requirement of strong convexity, we also present the convergence of Pseudo-linear TO for non-convex losses.
\begin{theorem}
\label{thm:linTVR_noncvx}
    Given Assumption, setting $\gamma_t = \gamma/(t+\mu)$ and $\beta_t = \beta/(t-1+\mu)$, where $\gamma$ and $\beta$ satisfy $\gamma < \frac{1}{L}$ and $\beta>1+ 6\gamma \sqrt{D_A^2 +L^2}$, and $\alpha_t = \alpha /(t-1+\mu)^2$ with $\sum_{s=1}^\infty \alpha_s < 1/D_w^2$, and $\mu$ to be large enough such that $\alpha_1 < 1, \beta_1<1$, $\gamma_1 < 1$, and $\alpha D_w^2 /\mu + \beta < \mu$, we have 
    \begin{eqnarray}
    \min_{1\leq t \leq T} \EE\left[\|\widehat G_t - G_t\|_2^2\right] \leq \OO\left(\frac{1}{\log T}\right)\nonumber\\
    \min_{1\leq t \leq T} \EE\left[\|G_t\|_2^2\right] \leq \OO\left(\frac{1}{\log T}\right).\nonumber
\end{eqnarray}
\end{theorem}

\subsection{Simplified Variants of Pseudo-linear TO}

\begin{algorithm}
\caption{Diagonal TO (Diag-TO)}
\label{alg:diagTVR}
\begin{algorithmic}[1]
\REQUIRE Learning rates $\alpha_t, \beta_t$, $\gamma_t$.
\STATE Initialize $w_1, a_0$ and $b_0$.
\FOR {$t=1,\ldots, T$}
\STATE Sample $\BB_t \subset \{1,\ldots,N\}$ such that $|\BB_t|=b$
\STATE $g_t\gets \nabla F^{\BB_t}(w_t)$
\STATE $a_t \gets a_{t-1} + \alpha_t(g_t -a_{t-1}\odot w_t - b_{t-1})\odot w_t$
\STATE $b_t \gets b_{t-1} + \beta_t(g_t - a_{t-1}\odot w_t - b_{t-1})$
\STATE $\widehat G_t \gets a_t \odot w_t + b_t$
\STATE $w_{t+1} \gets w_t - \gamma_t \widehat G_t $
\ENDFOR
\end{algorithmic}
\end{algorithm}

\begin{algorithm}
\caption{RankOne TO (RO-TO)}
\label{alg:rankoneTVR}
\begin{algorithmic}[1]
\REQUIRE  Learning rates $\alpha_t, \beta_t$, $\gamma_t$.
\STATE Initialize $w_1, a_0, c_0$ and $b_0$.
\FOR {$t=1,\ldots, T$}
\STATE Sample $\BB_t \subset \{1,\ldots,N\}$ such that $|\BB_t|=b$
\STATE $g_t\gets \nabla F^{\BB_t}(w_t)$
\STATE $a_t \gets a_{t-1} + \alpha_t(g_t -a_{t-1}c_{t-1}^\top w_t - b_{t-1}) c_{t-1}^{\top}w_t$
\STATE $c_t \gets c_{t-1} + \alpha_t(g_t -a_{t-1}c_{t-1}^\top w_t - b_{t-1})^\top a_{t-1} w_t$
\STATE $b_t \gets b_{t-1} + \beta_t(g_t - a_{t-1}c_{t-1}^\top w_t - b_{t-1})$
\STATE $\widehat G_t \gets a_t c_t^\top w_t + b_t$
\STATE $w_{t+1} \gets w_t - \gamma_t \widehat G_t$
\ENDFOR
\end{algorithmic}
\end{algorithm}

A notable limitation of Pseudo-linear TO lies in its requirement of memory. For model weights $w\in\RR^d$, the optimizer variables $A$ and $b$ introduce $d^2+d$ new parameters, which is impractical for large models, when $d$ is large. To address this challenge, we propose two memory-efficient variants of the fully parameterized Pseudo-linear TO (Full-TO) introduced in Algorithm~\ref{alg:linearTVR}.
\begin{itemize}
    \item \textbf{Diagonal TO} (Algorithm~\ref{alg:diagTVR}) enforces $A_t = \mathrm{diag}\{a_t\}$, where $a_t\in\RR^d$, reducing the parameter count to $\OO(d)$. Here, $\odot$ denotes element-wise vector multiplication.
    \item \textbf{RankOne TO} (Algorithm~\ref{alg:rankoneTVR}) employs factorization $A_t = a_t c_t^\top$ with $a_t, c_t\in\RR^d$, similarly achieving memory complexity $\OO(d)$ while maintaining a richer parameter interaction.
\end{itemize}

The gradient update rules in both algorithms (Lines 5-6 of Algorithm~\ref{alg:diagTVR} and Lines 5-7 of Algorithm~\ref{alg:rankoneTVR}) are derived from the corresponding structured gradient computations for each parameterization. These modifications enlarge the scale of problems that are suitable under the Pseudo-linear TO framework. It is easy to extend the proof of Theorem~\ref{thm:linTVR} and show the convergence of Diagonal TO and RankOne TO. 

\section{Experiments}
\subsection{Datasets and Implementations}
We conduct comprehensive empirical evaluations comparing various TO implementations against ADAM, Momentum and SAGA across multiple real-world datasets.

\begin{itemize}
\item \textbf{MNIST} \citep{deng2012mnist}: A handwritten digit dataset comprising of 60,000 grayscale images (28×28 pixels)
\item \textbf{News20} \citep{Lang95}: Approximately 18,000 newsgroup posts across 20 topics
\item \textbf{CovType} \citep{Blackard1998CoverType}: Forest cover type prediction dataset with 581,012 samples and 54 features across 7 classes
\item \textbf{SmallNorb} \citep{LeCun2004LearningMF}: 3D object recognition dataset containing images of 50 toys from 5 categories
\item \textbf{KDD} \citep{tavallaee2009detailed}: Network intrusion detection subset from KDD CUP 99
\item \textbf{CIFAR-10} \citep{krizhevsky2009learning}: 60,000 color images across 10 object classes
\item \textbf{Alpaca} \citep{alpaca}: 52,000 instruction-following demonstrations generated by text-davinci-003
\end{itemize}

The models employed are logistic regressions or deep neural networks. The architectures are described later in the appendix. The last dataset corresponds to fine-tuning of an LLM. The experiments span strongly convex, convex, non-convex and LLM tasks.

Given the negligible computational overhead beyond gradient computation, we focus on convergence speed measured in epochs. For each experiment, we perform extensive hyperparameter tuning:
\begin{itemize}
    \item ADAM learning rates: $\gamma \in [10^{-5}, \;2\times10^{-5},\; 5\times 10^{-5},\; 10^{-4} ,\; 2\times 10^{-4},\; 5\times 10^{-4},\; 10^{-3},\; 2\times 10^{-3},\; 5\times10^{-3}]$, 
    \item TO methods: $\gamma \in [10^{-3},\; 2\times 10^{-3},\; 5\times10^{-3},\; 0.01,\; 0.02,\; 0.05,\; 0.1,\; 0.2,\; 0.5]$,
    \item TO-specific parameters: $\alpha, \beta \in$ [0.0, 0.01, 0.1, 0.5, 1.0].
\end{itemize}
All the hyperparameters are tested under constant and exponentially decaying scheme. The exponential decay rates are selected among $[0.6,\; 0.8,\; 0.95]$. The trainable weights of logistic regression and FFN are initialized from standard normal sampling. The mini-batch size is set to be 64. 

We use the full training loss as our primary evaluation metric. To ensure statistical reliability, we repeat each experiment five times with different random seeds. Using ADAM as our baseline, we calculate the relative performance difference $\rho$ for each method. Let $\text{loss}_k$ be the training loss of a considered algorithm in epoch $k$. The relative difference is defined as
\begin{eqnarray}
    \rho = \text{avg}\left(\frac{\min_{1\leq k\leq T} \{\text{ADAM loss}_k\} - \min_{1\leq k\leq T} \{\text{loss}_k\}}{\min_{1\leq k\leq T} \{\text{ADAM loss}_k\}}\right),\nonumber
\end{eqnarray}
where $T$ is the number of epochs and the average is taken across multiple runs. Metric $\rho > 0$ indicates superior performance to ADAM in terms of achieved minimum loss. We complement this with two-sided Wald tests on the relative differences, reporting significance levels $s$. Our statistical analysis follows these conventions:
\begin{itemize}
    \item $s<0.05$: Significantly better than ADAM
    \item $s > 0.95$: Significantly worse than ADAM
    \item $0.05 \leq s \leq 0.95$: No significant difference from ADAM.
\end{itemize}

Our theoretical analyses show that pseudo-linear and other variants of TO achieve an $\OO(1/t)$ convergence rate for strongly convex losses, and a convergence of the approximate variance with the same convergence rate. This suggests the expectation of TO's advantage on strongly convex task and noisy and complex tasks.

\subsection{Convex Experiments}
\label{sec:convex_experiments}
For convex losses, we train logistic regression models on five of the previously mentioned datasets. Since our theoretical analyses primarily address strongly convex losses, we emphasize results using logistic regression with cross-entropy loss and L2 regularization. Our experimental protocol involves three key validation steps to determine the optimal regularization parameter $\lambda$. First, we randomly partition each dataset into training and validation sets. Second, we train models using ADAM with $\lambda$ values from the discrete set $[0,\; 5\times 10^{-4},\; 10^{-3},\; 5\times 10^{-3},\; 0.01,\; 0.05]$, selecting the $\lambda$ yielding best validation performance with ADAM. When the optimal $\lambda$ equals zero, we conduct additional evaluations by randomly sampling five normally distributed $\lambda$ values with mean being 0 and standard deviation of 1, and filter the negative values, in order to thoroughly compare TO variants against ADAM on strongly convex losses. The final selected $\lambda$ values for each dataset appear in the appendix. 

\begin{table}[h]
\centering
\begin{tabular}{r|r|r|r|r}
Dataset & $\lambda$ & Full-TO & Diag-TO & RO-TO \\ \hline
SmallNorb & $\lambda^*$ & N/A & $\mathbf{18.0}\;(+)$ & $13.0\;(+)$ \\ \hline
CovType & $\lambda^*$ & $\mathbf{3.3}\;(\sim)$ & $-2.0\;(\sim)$ & $-2.1\;(\sim)$ \\ \hline
\multirow{5}{*}{KDD} & $\lambda_1$ & $15.0\;(+)$ & $\mathbf{16.0}\;(+)$ & $14.0\;(+)$ \\
 & $\lambda_2$ & $6.0\;(+)$ & $\mathbf{5.8}\;(+)$ & $-3.0\;(\sim)$ \\
 & $\lambda_3$ & $10.0\;(+)$ & ${10.0\;(+)}$ & $\mathbf{10.2}\;(+)$ \\
 & $\lambda_4$ & $\textbf{9.0}\;(+)$ & $7.6\;(+)$ & $-6.0\;(\sim)$ \\
 & $\lambda_5$ & $11.0\;(+)$ & $\mathbf{11.2}\;(+)$ & $8.7\;(+)$ \\ \hline
\multirow{5}{*}{News20} & $\lambda_1$ & N/A & $\mathbf{0.4}\;(+)$ & $0.4\;(+)$ \\
 & $\lambda_2$ & N/A & $\mathbf{0.2}\;(+)$ & $0.2\;(+)$ \\
 & $\lambda_3$ & N/A & $\mathbf{4.4}\;(+)$ & $4.3\;(+)$ \\
 & $\lambda_4$ & N/A & $\mathbf{4.8}\;(+)$ & $4.8\;(+)$ \\
 & $\lambda_5$ & N/A & $\mathbf{16.8}\;(+)$ & $16.7\;(+)$ \\ \hline
\multirow{5}{*}{MNIST} & $\lambda_1$ & $\mathbf{4.2}\;(+)$ & $3.9\;(+)$ & $-1.1\;(-)$ \\
 & $\lambda_2$ & $\mathbf{4.4}\;(+)$ & $3.9\;(+)$ & $-1.0\;(-)$ \\
 & $\lambda_3$ & $-7.1\;(-)$ & $\mathbf{0.7}\;(+)$ & $\mathbf{0.7}\;(+)$ \\
 & $\lambda_4$ & $\mathbf{0.3}\;(+)$ & $0.2\;(+)$ & $0.2\;(+)$ \\
 & $\lambda_5$ & $\mathbf{0.7}\;(+)$ & $0.6\;(+)$ & $-3.4\;(-)$
\end{tabular}
\caption{Values of $\rho$ for strongly convex tasks (all values $\times 10^{-3}$). Full-TO denotes for full version of pseudo-linear TO in Algorithm~\ref{alg:linearTVR}. Diag-TO and RO-TO are short for the two memory efficient variants of TO described in Algorithm~\ref{alg:diagTVR} and Algorithm~\ref{alg:rankoneTVR}. `+', `$\sim$' and `-' denote for $s< 0.05$, $0.05\leq s\leq 0.95$, and $s>0.95$, respectively.}
\label{tab:strcvxresult}
\end{table}

\begin{figure}[!t]
    \centering

    \begin{subfigure}[b]{0.8\linewidth}
        \includegraphics[width=\linewidth]{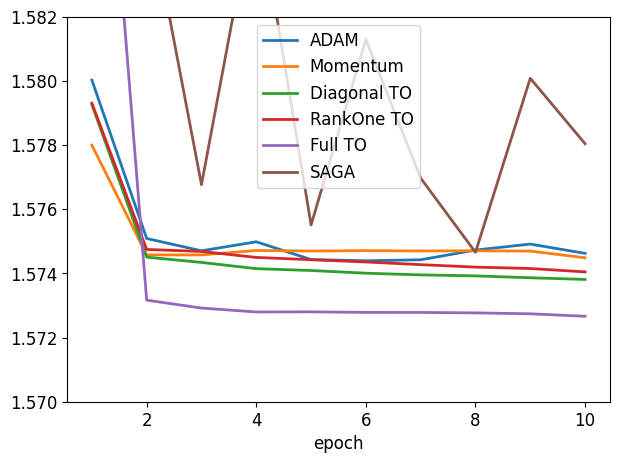}
        \caption{CovType}
    \end{subfigure}

    \begin{subfigure}[b]{0.8\linewidth}
        \includegraphics[width=\linewidth]{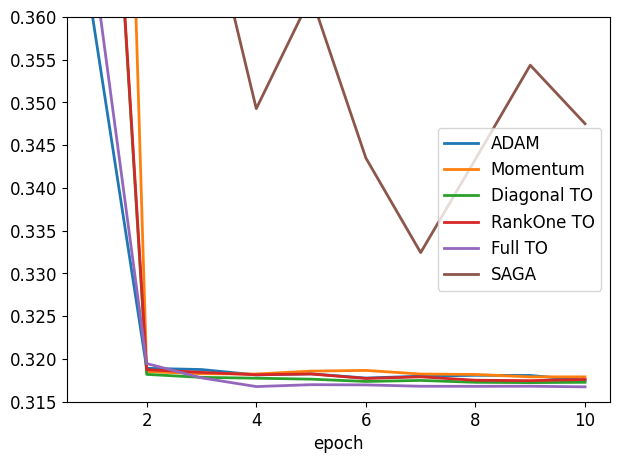}
        \caption{KDD}
    \end{subfigure}
    
    \begin{subfigure}[b]{0.8\linewidth}
        \includegraphics[width=\linewidth]{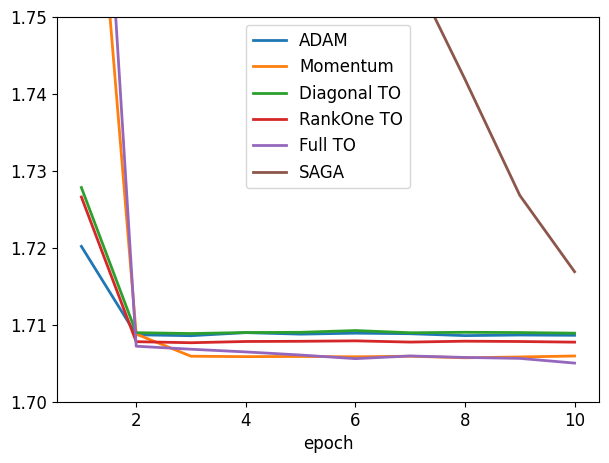}
        \caption{MNIST}
    \end{subfigure}
    \caption{Training loss curves for logistic regression ($\lambda=0$) on different datasets}
    \label{fig:convex}
\end{figure}

Table~\ref{tab:strcvxresult} presents the values of $\rho$ and significance for all strongly convex experiments. Due to memory constraints, we exclude the full version of pseudo-linear TO results for SmallNorb and News20 datasets. For experiments in which Full-TO is feasible, we report the memory requirements of Full-TO and its two memory-efficient variants in Table~\ref{tab:memory}. The results show that Diagonal TO and RankOne TO substantially reduce the memory requirements compared to Full-TO, making them more suitable for high-dimensional optimization problems. Our analysis of 17 strongly convex experiments reveals that TO variants significantly outperform ADAM in 16 cases. TO variants also outperform Momentum in 15 cases. The complete values of $\rho$ including Momentum are in the appendix.

\begin{table}[]
\centering
\begin{tabular}{l|l|l|l}
                 & Full-TO  & Diag-TO  & RO-TO    \\ \hline
Logistic-CovType & 150 KB   & 1.5 KB   & 2.3 KB   \\ \hline
Logistic-MNIST   & 400.8 MB & 80.1 KB  & 120.1 KB \\ \hline
FFN-CovType      & 381.6 MB & 78.1 KB  & 117.2 KB \\ \hline
FFN-MNIST        & 39.9 GB  & 808.7 KB & 1.2 MB  
\end{tabular}
\caption{Memory requirements of Full-TO and its memory-efficient variants}
\label{tab:memory}
\end{table}

To investigate performance on convex but non-strongly-convex losses, we conduct parallel experiments without regularization $\lambda = 0$ on the same datasets. In Figure~\ref{fig:convex}, we plot the curve of training loss of the $\lambda=0$ experiments for the three datasets where Full TO could be evaluated (CovType, MNIST and KKD) and find that pseudo-linear TO exceeds the performances of ADAM, Momentum and SAGA. Additional results, such as training losses for different additional datasets, and the standard deviation of the train loss among different random seeds can be found in the appendix. Our empirical findings suggest initializing with $\alpha \approx 0.01$ while $\beta\approx 1$. Exponential decaying of $\alpha$ and $\beta$ also plays crucial rules practically. These settings provide robust performance across different datasets.

\subsection{Nonconvex Experiments}

\begin{figure}[!t]
    \centering
    \begin{subfigure}[b]{0.8\linewidth}
        \includegraphics[width=\linewidth]{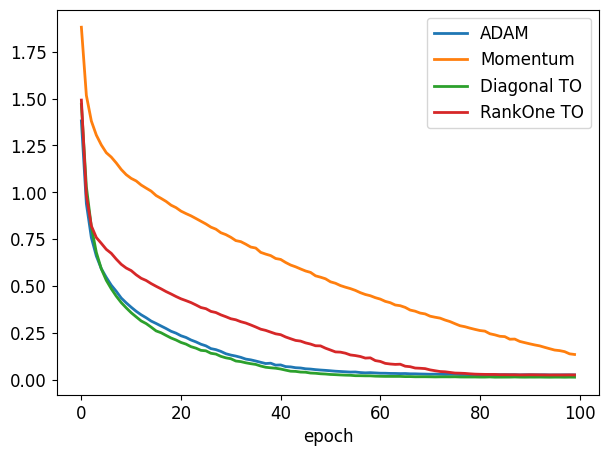}
        \caption{ResNet18-CIFAR10}
    \end{subfigure}
    
    \begin{subfigure}[b]{0.8\linewidth}
        \includegraphics[width=\linewidth]{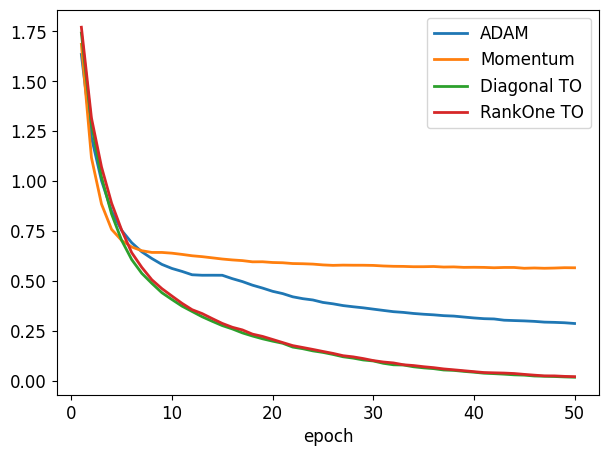}
        \caption{ResNet50-CIFAR10}
    \end{subfigure}
    
    \begin{subfigure}[b]{0.8\linewidth}
        \includegraphics[width=\linewidth]{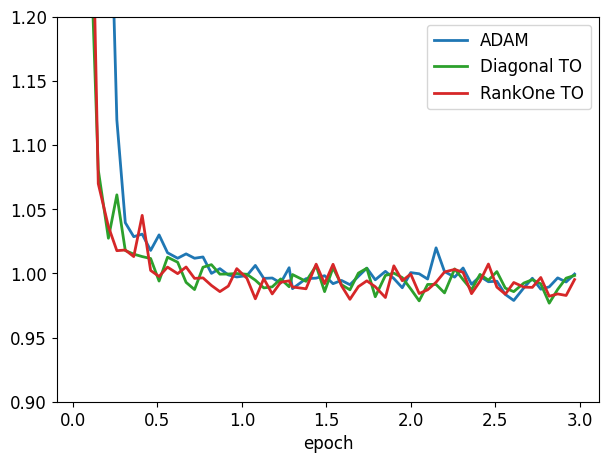}
        \caption{Llama7b-Alpaca}
    \end{subfigure}
    \caption{Training loss curves for non-convex tasks}
    \label{fig:noncvx}
\end{figure}
\begin{figure}[!t]
    \centering
    \begin{subfigure}[b]{0.8\linewidth}
        \includegraphics[width=\linewidth]{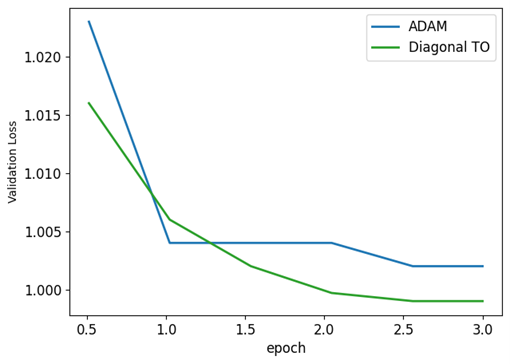}
    \end{subfigure}
    
    \caption{Validation loss curve for Llama7b-alpaca task}
    \label{fig:llm_val}
\end{figure}

For non-convex optimization tasks, we conduct three experimental evaluations. First, we train FFNs on both the CovType and KDD datasets. Second, we train ResNets of varying depths (18 and 50 layers) on the CIFAR-10 classification task. Third, we perform fine-tuning of a selected subset of pretrained weights from the Llama-7b model using the Alpaca dataset.

\begin{table}[h]
\centering
\setlength{\tabcolsep}{3pt} 
\begin{tabular}{r|r|r|r|r}
Model & Dataset & Full-TO & Diag-TO & RO-TO \\ \hline
\multirow{2}{*}{FFN} & CovType & $-1.3\times 10^{-4}$ & $-1.8\times 10^{-4}$ & $-2.0\times 10^{-4}$ \\
 & KDD & $-1.1\times 10 ^{-3}$ & $-3.2\times 10^{-3}$ & $-2.1\times 10^{-3}$ \\ \hline
ResNet18 & \multirow{2}{*}{CIFAR10} & N/A & $\mathbf{0.5}$ & $8.2\times 10^{-3}$ \\ \cline{1-1}
ResNet50 &  & N/A & $\mathbf{0.9}$ & $\mathbf{0.9}$ \\ \hline
Llama-7b & Alpaca & N/A & $\mathbf{2.1\times 10^{-3}}$ & $-9.2\times 10^{-4}$
\end{tabular}
\caption{Values of $\rho$ for non-convex tasks}
\label{tab:NonCvx}
\end{table}

The relative performance differences of TO variants are reported in Table~\ref{tab:NonCvx}. Figure~\ref{fig:noncvx} shows the curves of training loss for the tasks of training CIFAR10 with ResNet and finetuning Llama. These experiments demonstrate TO's superior performance on challenging non-convex problems. For the Llama fine-tuning task, TO achieves faster initial convergence compared to ADAM. In the ResNet classification tasks, TO consistently discovers solutions with better final optimality. To investigate the effect of TO on the generalization performance of LLMs, we evaluate the validation loss in the Llama-7b Alpaca experiment, with the results shown in Figure~\ref{fig:llm_val}. We observe that Diagonal TO consistently outperforms ADAM in terms of validation-loss convergence, particularly during the early stages of training. Moreover, Diagonal TO achieves a lower final validation loss than ADAM. Given its rapid loss reduction in the initial iterations, these results suggest that TO is especially well suited for LLM training scenarios where computational budgets or training time are limited. The training curves for FFN are in the appendix, which show comparable performance between TO and ADAM, with no statistically significant difference observed. 

\section{Conclusion}
This paper introduces Trainable Optimizer (TO), a novel optimization framework that jointly learns model parameters and optimization dynamics through a trainable linear gradient approximation, with theoretical guarantees showing simultaneous convergence of both parameters and approximate error. We develop memory-efficient Diagonal and RankOne TO variants that maintain $\OO(d)$ complexity while preserving convergence properties. Experiments across convex and non-convex tasks on seven datasets demonstrate TO's advantages: superior performance to ADAM in 16/17 strongly convex cases, faster convergence and better solutions for ResNet training and Llama fine-tuning, and comparable performance on simpler non-convex tasks.

\newpage

\bibliography{example_paper}
\bibliographystyle{icml2026}

\newpage
\appendix
\onecolumn
\section{Proofs}
\subsection{Technical Lemmas}
\begin{lemma}
\label{lem:BddW}
    For all $t$, $\left\|w_{t+1} - w_{t}\right\|_2 \leq \gamma_t \left\|\widehat G_t\right\|_2$.
\end{lemma}
\begin{proof}
\citep{reddi2019convergence} shows that for all $u$ and $u'$, $\|\Pi_\FF(u) - \Pi_\FF(u')\|_2\leq \|u-u'\|_2$. In our case, we have $w_t\in \FF$, therefore $\|w_{t+1} - w_t\|_2 = \|\Pi_{\FF}(w_t-\gamma_t\widehat G_t)-\Pi_{\FF}(w_t)\|_2 \leq \|\gamma_t \widehat G_t\|_2.$
\end{proof}

\begin{lemma}
    \label{lem:SolutiontoIneq} Let $0<k_2<k_1$  and $b_1<0<b_2$, for any $x_0$ and $y_0$. There exists $x\geq x_0$ and $y\geq y_0$ such that $k_2 x+b_2 \leq y \leq k_1 x+b_1$.
\end{lemma}
\begin{proof}
    Let $x^* = \frac{b_2-b_1}{k_1-k_2}$, $y^* = \frac{k_1b_2 -k_2b_1}{k_1-k_2}$, and let $x_n = x^* + n$, $y_n = y^* + \frac{k_1+k_2}{2}n$. It is easy to check that $k_2 x_n+b_2 \leq y_n \leq k_1 x_n+b_1$ for all $n>0$. We can pick $n$ large enough such that $x_n\geq x_0$ and $y_n\geq y_0$.
\end{proof} 

\subsection{Proof to Proposition~\ref{lem:Bddg}}
\begin{proof}
By definition $g_t$ reads
\begin{eqnarray}
    \|g_t\|_2 &=&\|\nabla F^{\BB_t}(w_t)\|_2 \nonumber\\
    &\leq& \|\nabla F^{\BB_t}(w_t) - \nabla F^{\BB_t}(w^*)\|_2 +\|\nabla F^{\BB_t}(w^*)\|_2\nonumber\\
    &\leq& L\|w_t-w^*\| + \max_{\BB}\|\nabla F^{\BB}(w^*)\|_2\nonumber\\
    &\leq& 2LD_w + \max_{\BB}\|\nabla F^{\BB}(w^*)\|_2,\nonumber
\end{eqnarray}
where the inequalities hold according to Assumption. The proposition is proved by letting $D_G = 2LD_w + \max_{\BB}\|\nabla F^{\BB}(w^*)\|_2$.
\end{proof}

\subsection{Proof to Proposition~\ref{lem:BddAb}}
\begin{proof}
    Consider 
    \begin{eqnarray}
        D_b &=& \max\left\{ \left\|b_0\right\|_2, \frac{1}{1-S_\alpha D_w^2} \left( D_G(1+S_\alpha D_w^2) + \left\|A_0\right\|_2 D_w\right)\right\}\nonumber\\
        D_A &=& \left\|A_0\right\|_2 + S_\alpha D_w(D_b+D_G)\nonumber.
    \end{eqnarray}
    
    Apparently $\|b_0\|_2 \leq D_b$. We assume that for all $0 \leq t \leq T-1$, $\|b_{t}\|\leq D_b$. Reformulate the updating rule for $A_t$ yields
    \begin{eqnarray}
        A_t &=& A_{t-1} + \alpha_t(g_t - A_{t-1}w_t - b_{t-1})w_t^\top\nonumber\\
        &=& A_{t-1}(I_d-\alpha_t w_t w_t^\top)  + \alpha_t (g_t - b_t)w_t^\top.\nonumber
    \end{eqnarray}
    Then, for all $1\leq t \leq T-1$,
    \begin{eqnarray}
        \left\|A_t\right\|_2 &=& \left\| A_{t-1}(I_d-\alpha_t w_t w_t^\top)  + \alpha_t (g_t - b_t)w_t^\top\right\|_2 \nonumber\\
        &\leq& \left\|I_d - \alpha_t w_t w_t^\top\right\|_2 \left\|A_{t-1}\right\|_2 + \alpha_t \left\|g_t - b_t\right\|_2 \left\| w_t\right\|_2 \nonumber\\
        &\leq& \left\|A_{t-1}\right\|_2 + \alpha_t (D_b+D_G)D_w,\nonumber
    \end{eqnarray}
    where the last inequality holds because $\|g_t\|\leq D_G$. We recursively obtain
    \begin{eqnarray}
        \left\|A_t\right\|_2 &\leq& \left\|A_0\right\|_2 + \sum_{s=1}^t \alpha_s (D_b+D_G)D_w\nonumber\\
        &<& \left\|A_0\right\|_2 + S_\alpha (D_b +D_G)D_w =D_A.\nonumber
    \end{eqnarray}
    
    Furthermore, we have
    \begin{eqnarray}
        \left\|b_T\right\|_2 &=& \left\|(1-\beta_{T}) b_{T-1} + \beta_T(g_T - A_{T-1} w_T )\right\|_2 \nonumber\\
        &\leq& (1-\beta_T) \|b_{T-1}\|_2 + \beta_T (\|g_T\|_2 + \|A_{T-1}\|_2\|w_T\|_2)\nonumber\\
        &\leq& (1-\beta_T) D_b + \beta_T (D_G + D_AD_w)\nonumber\\
        &\leq& D_b.\nonumber
    \end{eqnarray}
    The last inequality holds because
    \begin{eqnarray}
        D_b &=& \max\left\{ \left\|b_0\right\|_2, \frac{1}{1-S_\alpha D_w^2} \left( D_G(1+S_\alpha D_w^2) + \left\|A_0\right\|_2 D_w\right)\right\}\nonumber\\
        &\geq& \frac{1}{1-S_\alpha D_w^2} \left( D_G(1+S_\alpha D_w^2) + \left\|A_0\right\|_2 D_w\right)\nonumber,
    \end{eqnarray}
    which implies that 
    \begin{eqnarray}
        (1-S_\alpha D_w^2)D_b &\geq& D_G(1+S_\alpha D_w^2) +\left\|A_0\right\|_2 D_w,\nonumber
    \end{eqnarray}
    and therefore
    \begin{eqnarray}
        D_b &\geq& D_b S_\alpha D_w^2 + D_G(1+S_\alpha D_w^2) +\left\|A_0\right\|_2 D_w\nonumber\\
        &=&D_G + D_w(\|A_0\|_2 + S_\alpha D_w(D_G+D_b))=D_G+D_wD_A.\nonumber
    \end{eqnarray}
    According to the principle of induction, this finishes the proof.
\end{proof}

\subsection{Proof to Theorem~\ref{thm:linTVR}}
\begin{proof}
We start with bounding the approximation loss of the full gradient. Letting $G_t=\nabla F(w_t)$, the approximation error reads
\begin{eqnarray}
    \widehat G_t - G_t &=& A_{t} w_t + b_t - G_t \nonumber\\
    &=& \left(A_{t-1} +\alpha_t \left( g_t - A_{t-1} w_t -b_{t-1}\right)w_t^{\top} \right) w_t+ b_{t-1} + \beta_t \left(g_t - A_{t-1} w_t - b_{t-1}\right) -G_t\nonumber\\
    &=& A_{t-1} w_t + b_{t-1} + (\alpha_t \|w_t\|_2^2 + \beta_t) \left(g_t -A_{t-1} w_t - b_{t-1}\right) - G_t\nonumber\\
    &=& (1-\delta_t) \left(A_{t-1} w_t + b_{t-1} - G_t\right) + \delta_t \left(g_t - G_t\right),\nonumber
\end{eqnarray}
where $\delta_t = \alpha_t \left\|w_t\right\|_2^2 + \beta_t$. Letting $\delta = \alpha D_w^2 /\mu+\beta$, apparently we have $\delta_t = (\alpha\|w_t\|_2^2 / (t-1+\mu) + \beta)/(t-1+\mu)  < \delta/(t-1+\mu)$. Taking the square norm of the approximation error yields
\begin{eqnarray}
    \left\|\widehat G_t - G_t\right\|_2^2 &=& (1-\delta_t)^2 \left\|A_{t-1} w_t + b_{t-1} - G_t\right\|_2^2 +\delta_t^2 \left\|g_t - G_t\right\|_2^2\nonumber\\
    && + 2(1-\delta_t)\delta_t \left(A_{t-1}w_t + b_{t-1} - G_t\right)^{\top}(g_t- G_t).\nonumber
\end{eqnarray}
Letting $\EE_t[\cdot] = \EE\left[\cdot \left|\BB_{1},\cdots,\BB_{t-1}\right.\right]$, we notice that $A_{t-1}, b_{t-1}, G_t$ and $w_t$ are known given $\BB_{1},\cdots,\BB_{t-1}$. Then, we have

\begin{eqnarray}
    \EE_{t}\left[\left\|\widehat G_t - G_t\right\|_2^2\right] &=& (1-\delta_t)^2 \left\|A_{t-1} w_t + b_{t-1} - G_t\right\|_2^2 +\delta_t^2 \EE_t\left[\left\|g_t - G_t\right\|_2^2\right]\nonumber\\
    && + 2(1-\delta_t)\delta_t \left(A_{t-1}w_t + b_{t-1} - G_t\right)^{\top}\EE_t\left[g_t- G_t\right]\nonumber\\
    &=&(1-\delta_t)^2 \left\|A_{t-1} w_t + b_{t-1} - G_t\right\|_2^2 +\delta_t ^2\EE_t\left[\left\|g_t - G_t\right\|_2^2\right]\nonumber\\
    &\leq& (1-\delta_t)^2 \left\|A_{t-1} w_t + b_{t-1} - G_t\right\|_2^2 + D_V\frac{\delta^2}{(t-1+\mu)^2}\nonumber\\
    &=& (1-\delta_t)^2\underbrace{ \left\|\widehat G_{t-1} - G_{t-1} + A_{t-1}(w_t-w_{t-1}) -\left( G_{t}-G_{t-1}\right)\right\|_2^2}_{T_1} + D_V\frac{\delta^2}{(t-1+\mu)^2},\nonumber
\end{eqnarray}
where the second equality holds because $\EE_t g_t = G_t$.

For three arbitrary vectors $a_1, a_2, a_3$, the Cauchy inequality implies that for any $\lambda > 0$,
\begin{eqnarray}
    \|a_1 + a_2 + a_3\|_2^2 &=&\|a_1\|_2^2 + \|a_2\|_2^2 + \|a_3\|_2^2 + 2a_1^\top a_2+2a_2^\top a_3 + 2 a_3^\top a_1\nonumber\\
    &\leq& \|a_1\|_2^2 + \|a_2\|_2^2 + \|a_3\|_2^2 + \frac{1}{2\lambda}\|a_1\|^2_2 + 2\lambda \| a_2\|_2^2+\frac{1}{2\lambda}\|a_1\|^2_2 + 2\lambda \| a_3\|_2^2+ \|a_2\|_2^2 + \|a_3\|_2^2\nonumber\\
    &=&(1+\frac{1}{\lambda}) \|a_1\|_2^2 + (2+2\lambda)\|a_2\|_2^2 + (2+2\lambda) \|a_3\|_2^2.\nonumber
\end{eqnarray}
Applying this inequality to $T_1$ and replacing $\lambda$ with $\lambda/(t-1+\mu)$, we have
\begin{eqnarray}
    T_1 &\leq& \left(1+\frac{\lambda}{t-1+\mu}\right) \left\|\widehat G_{t-1}-G_{t-1}\right\|_2^2 + 2\left(1+\frac{t-1+\mu}{\lambda}\right) \left(\|A_{t-1}(w_t-w_{t-1})\|_2^2 +\left\|G_{t}-G_{t-1}\right\|_2^2\right)\nonumber\\
    &\leq& \left(1+\frac{\lambda}{t-1+\mu}\right) \left\|\widehat G_{t-1}-G_{t-1}\right\|_2^2 + 2\left(1+\frac{t-1+\mu}{\lambda}\right)\left(D_A^2 + L^2\right)\|w_{t}-w_{t-1}\|_2^2\nonumber\\
    &\leq& \left(1+\frac{\lambda}{t-1+\mu}\right) \left\|\widehat G_{t-1}-G_{t-1}\right\|_2^2 + 2\left(1+\frac{t-1+\mu}{\lambda}\right)\left(D_A^2 + L^2\right)\gamma_{t-1}^2\|\widehat G_{t-1}\|_2^2\nonumber\\
    &\leq& \left(1+\frac{\lambda}{t-1+\mu}\right) \left\|\widehat G_{t-1}-G_{t-1}\right\|_2^2 + 4\left(1+\frac{t-1+\mu}{\lambda}\right)\left(D_A^2 + L^2\right)\gamma_{t-1}^2(\|\widehat G_{t-1} - G_{t-1}\|_2^2 + \| G_{t-1}\|_2^2)\nonumber\\
    &=& \left(1+\frac{\lambda}{t-1+\mu} + 4\left(1+\frac{t-1+\mu}{\lambda}\right)\left(D_A^2 + L^2\right)\gamma_{t-1}^2\right) \left\|\widehat G_{t-1}-G_{t-1}\right\|_2^2 \nonumber\\
    &&+ 4\left(1+\frac{t-1+\mu}{\lambda}\right)\left(D_A^2 + L^2\right)\gamma_{t-1}^2\| G_{t-1}\|_2^2\nonumber,
\end{eqnarray}
where the second inequality applies Proposition~\ref{lem:BddAb} and Lipschitz continuity of $\nabla F$, the third inequality applies Lemma~\ref{lem:BddW} and the forth inequality further applies the Cauchy inequality. We let $\lambda = 2\gamma\sqrt{D_A^2+L^2}$. The upper bound is simplified to
\begin{eqnarray}
    T_1 &\leq& \left(1 + \frac{4\gamma \sqrt{D_A^2+L^2}}{t-1+\mu} + \frac{4\gamma^2(D_A^2+L^2)}{(t-1+\mu)^2}\right)\left\|\widehat G_{t-1} - G_{t-1}\right\|_2^2 + 4\left(1+\frac{t-1+\mu}{2\gamma\sqrt{D_A^2+L^2}}\right)\frac{\gamma^2(D_A^2+L^2)}{(t-1+\mu)^2}\|G_{t-1}\|_2^2\nonumber\\
    &=&\left(1 + \frac{2\gamma \sqrt{D_A^2+L^2}}{t-1+\mu}\right)^2\left\|\widehat G_{t-1} - G_{t-1}\right\|_2^2 + 4\left(1+\frac{t-1+\mu}{2\gamma\sqrt{D_A^2+L^2}}\right)\frac{\gamma^2(D_A^2+L^2)}{(t-1+\mu)^2}\|G_{t-1}\|_2^2\label{eqn:T1}\\
    &\leq&\left(1 + \frac{2\gamma \sqrt{D_A^2+L^2}}{t-1+\mu}\right)^2\left\|\widehat G_{t-1} - G_{t-1}\right\|_2^2 + 4\left(1+\frac{t-1+\mu}{2\gamma\sqrt{D_A^2+L^2}}\right)\frac{\gamma^2L^2(D_A^2+L^2)}{(t-1+\mu)^2}\|w_{t-1} - w^*\|_2^2\nonumber\\
    &\leq&\left(1 + \frac{2\gamma \sqrt{D_A^2+L^2}}{t-1+\mu}\right)^2\left\|\widehat G_{t-1} - G_{t-1}\right\|_2^2 + \frac{4\gamma L^2\sqrt{D_A^2 +L^2}}{t-1+\mu}\|w_{t-1} - w^*\|_2^2,\nonumber
\end{eqnarray}
where the second inequality further applies the Lipschitz continuity of $\nabla F$. The last inequality holds because $\mu$ is large enough such that 
$$\frac{1}{t-1+\mu} \leq \frac{1}{\beta}<\frac{1}{2\gamma \sqrt{D_A^2 + L^2}},$$
and thus
$$\frac{1}{t-1+\mu} \left(1 + \frac{t-1+\mu}{2\gamma\sqrt{D_A^2+L^2}}\right) < \frac{1}{\gamma\sqrt{D_A^2+L^2}}.$$

The expectation $\EE_t[\|\widehat G_t - G_t\|_2^2]$ is upper bounded by
\begin{eqnarray}
    \EE_t\left[\left\|\widehat G_t - G_t\right\|_2^2\right] &\leq& \left(1 - \delta_t\right)^2 \left(1+\frac{2\gamma\sqrt{D_A^2+L^2}}{t-1+\mu}\right)^2\left\|\widehat G_{t-1} - G_{t-1}\right\|_2^2+\frac{C_1\gamma}{t-1+\mu}\left\|w_{t-1}-w^*\right\|_2^2 + \frac{D_V\delta^2}{(t-1+\mu)^2},\nonumber
\end{eqnarray}
where $C_1 = 4L^2\sqrt{D_A^2+L^2}$. The setting of parameters satisfies that $\beta_t < \delta_t < 1$ and therefore 
\begin{eqnarray}
    (1-\delta_t)^2 \left(1+ \frac{2\gamma\sqrt{D_A^2+L^2}}{t-1+\mu}\right)^2 &<& \left(1-\frac{\beta}{t-1+\mu}\right)^2\left(1+ \frac{2\gamma\sqrt{D_A^2+L^2}}{t-1+\mu}\right)^2\nonumber\\
    &<&\left( 1 - \frac{\beta-2\gamma\sqrt{D_A^2+L^2}}{t-1+\mu}\right)^2 <1 - \frac{\beta-2\gamma\sqrt{D_A^2+L^2}}{t-1+\mu}:=1-\frac{\beta^*}{t-1+\mu}\nonumber,
\end{eqnarray}
where $\beta^*:=\beta - 2\gamma\sqrt{D_A^2+L^2}>0$. We obtain
\begin{eqnarray}
     \EE_t\left[\left\|\widehat G_t - G_t\right\|_2^2\right] &\leq& \left(1-\frac{\beta^*}{t-1+\mu}\right) \left\|\widehat G_{t-1} - G_{t-1}\right\|_2^2 + \frac{C_1\gamma}{t-1+\mu} \left\|w_{t-1}-w^*\right\|_2^2 + \frac{D_V\delta}{(t-1+\mu)^2}.\nonumber
\end{eqnarray}
Taking full expectation on both side yields
\begin{eqnarray}
\label{Eq:GError}
     \EE\left[\left\|\widehat G_t - G_t\right\|_2^2\right] &\leq& \left(1-\frac{\beta^*}{t-1+\mu}\right) \EE\left [\left\|\widehat G_{t-1} - G_{t-1}\right\|_2^2 \right]+ \frac{C_1\gamma}{t-1+\mu}\EE\left[ \left\|w_{t-1}-w^*\right\|_2^2\right] + \frac{D_V\delta}{(t-1+\mu)^2}.
\end{eqnarray}

Next, we consider the distance between the iterate and the optimal point. We obtain
\begin{eqnarray}
    \left\|w_{t} - w^*\right\|_2^2 &\leq& \left\| w_{t-1} - \gamma_{t-1} \widehat G_{t-1} - w^* \right\| _2^2\nonumber\\
    &=& \left\|w_{t-1} - w^*\right\|_2^2 + \gamma_{t-1} ^2\left\|\widehat G_{t-1}\right\|_2^2 -2\gamma_{t-1} \widehat G_{t-1}^\top (w_{t-1}-w^*)\nonumber\\
    &\leq&\left\|w_{t-1} - w^*\right\|_2^2 + \gamma_{t-1} ^2 D_G^2 -2\gamma_{t-1} (\widehat G_{t-1} - G_{t-1}) ^{\top}(w_{t-1}-w^*) -2\gamma_{t-1}G_{t-1}^\top (w_{t-1}-w^*)\nonumber\\
    &\leq& (1-2c\gamma_{t-1})\left\|w_{t-1}-w^*\right\|_2^2 -2\gamma_{t-1} (\widehat G_{t-1}-G_{t-1})^\top (w_{t-1}-w^*) +\gamma_{t-1}^2 D_G^2\nonumber\\
    &\leq& (1-2c\gamma_{t-1})\left\|w_{t-1}-w^*\right\|_2^2 +\gamma_{t-1} c\left\|w_{t-1}-w^*\right\|_2^2 + \frac{\gamma_{t-1}}{c}\left\|\widehat G_{t-1} -G_{t-1}\right\|_2^2 + \gamma_{t-1} ^2 D_G^2\nonumber\\
    &=& \left(1-\frac{c\gamma}{t-1+\mu}\right) \left\|w_{t-1}-w^*\right\|_2^2 +\frac{\gamma}{c(t-1+\mu)} \left\|\widehat G_{t-1} - G_{t-1}\right\|_2^2 + \frac{D_G^2\gamma^2}{(t-1+\mu)^2},\nonumber
\end{eqnarray}
where the third inequality holds because of strong convexity of $F(w)$. Taking expectation we have
\begin{eqnarray}
    \label{Eq:wError}
    \EE\left[\left\|w_{t}-w^*\right\|_2^2\right] &\leq& \left(1-\frac{\gamma c}{t-1+\mu}\right) \EE\left[ \left\|w_{t-1}-w^*\right\|_2^2\right] + \frac{\gamma}{c(t-1+\mu)} \EE\left[\left\|\widehat G_{t-1} - G_{t-1}\right\|_2^2\right] \nonumber\\
    &&+ \frac{D_G^2\gamma^2}{(t-1+\mu)^2}.
\end{eqnarray}

The setting of hyperparameters guarantees that $\frac{\beta^* - 1}{C_1\gamma}> \frac{\gamma}{c(\gamma c - 1)} > 0$. Applying Lemma~\ref{lem:SolutiontoIneq}, there exist $M_1 \geq (\mu+1) \left\|w_1 - w^*\right\|_2^2$ and $M_2 \geq (\mu+1) \EE\left[\left\|\widehat G_1 - G_1\right\|_2^2\right]$, such that
\begin{eqnarray}
\label{Eq:M1M2}
    \frac{\gamma}{c(\gamma c -1)}M_2 + \frac{D_G^2\gamma^2}{\gamma c - 1} \leq M_1 \leq \frac{\beta^*-1}{C_1\gamma} M_2 - \frac{D_V\delta^2}{C_1\gamma}.
\end{eqnarray}
Apparently,
\begin{eqnarray}
    \EE\left[\left\|w_1-w^*\right\|_2^2\right]& \leq& \frac{M_1}{1+\mu},\nonumber\\
    \EE\left[ \left\|\widehat G_1 - G_1\right\|_2^2\right] &\leq& \frac{M_2}{1+\mu}.\nonumber
\end{eqnarray}
We use induction. Assuming
\begin{eqnarray}
    \EE\left[\left\|w_{t-1}-w^*\right\|_2^2\right]& \leq& \frac{M_1}{t-1+\mu},\nonumber\\
    \EE\left[ \left\|\widehat G_{t-1} - G_{t-1}\right\|_2^2\right] &\leq& \frac{M_2}{t-1+\mu},\nonumber
\end{eqnarray}
and according to (\ref{Eq:GError}) and (\ref{Eq:wError}), we obtain
\begin{eqnarray}
    \EE\left[\left\| \widehat G_t - G_t\right\|_2^2\right] &\leq& \left(1-\frac{\beta^*}{t-1+\mu}\right) \frac{M_2}{t-1+\mu} + \frac{C_1\gamma}{(t-1+\mu)^2}M_1 + \frac{D_V\delta^2}{(t-1+\mu)^2}\nonumber\\
    &=& \frac{t-2+\mu}{(t-1+\mu)^2}M_2 -\frac{1}{(t-1+\mu)^2}\left((\beta^* - 1)M_2 -C_1\gamma M_1 - D_V\delta^2 \right)\nonumber\\
    &\leq& \frac{M_2}{t+\mu} ,\nonumber
\end{eqnarray}
where the last inequality holds because $(t+\mu)(t-2+\mu)\leq(t-1+\mu)^2$ and the second term is guaranteed by (\ref{Eq:M1M2}) to be negative. We also have
\begin{eqnarray}
\EE\left[\left\|w_t-w^*\right\|_2^2\right] &\leq &\left(1-\frac{\gamma c}{t-1+\mu}\right) \frac{M_1}{t-1+\mu} + \frac{\gamma}{c(t-1+\mu)^2}M_2 + \frac{D_G^2\gamma^2}{(t-1+\mu)^2}\nonumber\\
&=& \frac{t-2+\mu}{(t-1+\mu)^2}M_1 - \frac{1}{(t-1+\mu)^2} \left((\gamma c -1)M_1-\frac{\gamma}{c}M_2 - D_G^2\gamma^2 \right)\nonumber\\
&\leq& \frac{M_1}{t+\mu}.\nonumber
\end{eqnarray}
This completes the proof.
\end{proof}

\subsection{Proof to Theorem~\ref{thm:linTVR_noncvx}}
\begin{proof}
Equation (\ref{eqn:T1}) also holds for non-convex loss functions. Similar to the proof to Theorem~\ref{thm:linTVR}, we further obtain
\begin{eqnarray}
\label{eqn:Diffiter}
    \EE\left[\left\|\widehat G_t - G_t\right\|_2^2\right] &\leq& \left(1-\frac{\beta^*}{t-1+\mu}\right) \EE\left [\left\|\widehat G_{t-1} - G_{t-1}\right\|_2^2 \right]+ \frac{C_2\gamma}{t-1+\mu}\EE\left[ \left\|G_{t-1}\right\|_2^2\right] + \frac{D_V\delta}{(t-1+\mu)^2},
\end{eqnarray}
where $\beta^* = \beta - 2\gamma \sqrt{D_A^2+L^2}$ and $C_2 = 4\sqrt{D_A^2+L^2}$. We apply the Lipschitz continuity of $\nabla F$ to obtain
\begin{eqnarray}
    F(w_{t}) - F(w_{t-1}) &\leq& -\gamma_{t-1} G_{t-1}^\top \widehat G_{t-1} + \frac{L\gamma_{t-1}^2\left\|\widehat G_{t-1}\right\|_2^2}{2}\nonumber\\
    &=& \frac{\gamma_{t-1}}{2} \|G_{t-1} - \widehat G_{t-1}\|_2^2 - \frac{\gamma_{t-1}}{2}\left\|G_{t-1}\right\|_2^2 +\frac{\gamma_{t-1}(L\gamma_{t-1} - 1)}{2}\|\widehat G_{t-1}\|_2^2 \nonumber\\
    &\leq& \frac{\gamma_{t-1}}{2} \|G_{t-1} - \widehat G_{t-1}\|_2^2 - \frac{\gamma_{t-1}}{2}\left\|G_{t-1}\right\|_2^2,\nonumber
\end{eqnarray}
where the last inequality holds because $\gamma_tL<1$. This implies
\begin{eqnarray}
\label{eqn:Giter}
    \frac{\gamma}{t-1+\mu}\EE\left[\left\|G_{t-1}\right\|_2^2\right] \leq \frac{\gamma}{t-1+\mu}\EE\left[\left\|\widehat G_{t-1} - G_{t-1} \right\|_2^2\right] + 2\EE\left[F(w_{t-1})\right] - 2\EE\left[F(w_{t})\right].
\end{eqnarray}
The summation over $t$ of (\ref{eqn:Diffiter}) reads
\begin{eqnarray}
    \sum_{t=2}^T \EE\left[\left\|\widehat G_t - G_t\right\|_2^2\right]&\leq& \sum_{t=2}^{T} \left(1-\frac{\beta^*}{t-1+\mu}\right) \EE\left[\left\|\widehat G_{t-1} - G_{t-1}\right\|_2^2\right] + C_2\sum_{t=2}^{T}\frac{\gamma}{t-1+\mu} \EE\left[\left\|G_{t-1}\right\|_2^2\right] + \sum_{t=2}^{T}\frac{D_V\delta}{(t-1+\mu)^2}\nonumber\\
    &=& \sum_{t=1}^{T-1} \left(1-\frac{\beta^*}{t+\mu}\right) \EE\left[\left\|\widehat G_t - G_t\right\|_2^2\right] + C_2\sum_{t=1}^{T-1}\frac{\gamma}{t+\mu} \EE\left[\left\|G_{t}\right\|_2^2\right] + \sum_{t=1}^{T-1}\frac{D_V\delta}{(t+\mu)^2}\nonumber\\
    &\leq& \sum_{t=1}^{T-1} \left(1-\frac{\beta^*}{t+\mu}\right) \EE\left[\left\|\widehat G_t - G_t\right\|_2^2\right] + C_2\sum_{t=1}^{T-1}\frac{\gamma}{t+\mu} \EE\left[\left\|\widehat G_t -G_{t}\right\|_2^2\right] \nonumber\\
    &&+ 2(\EE[F(w_1)] -\EE [F(w_T)]) + \sum_{t=1}^{T-1}\frac{D_V\delta}{(t+\mu)^2}.\nonumber
\end{eqnarray}
Given that $\beta > 6\gamma \sqrt{D_A^2+L^2}$, this implies
\begin{eqnarray}
    \sum_{t=1}^{T-1} \frac{\beta^*-\gamma C_2}{t+\mu} \EE\left[\left\|\widehat G_t - G_t\right\|_2^2\right] &\leq& \left(\sum_{t=1}^{T-1}\EE\left[\left\|\widehat G_t - G_t\right\|_2^2\right]-\sum_{t=2}^T\EE\left[\left\|\widehat G_t - G_t\right\|_2^2\right]\right) + 2(\EE[F(w_{1})] -\EE[F(w_T)]) + \sum_{t=1}^{T-1} \frac{D_V\delta}{(t+\mu)^2}\nonumber\\
    &=& \EE\left[\left\|\widehat G_1-G_1\right\|_2^2\right] - \EE\left[\left\|\widehat G_T-G_T\right\|_2^2\right] + 2(\EE[F(w_{1})] -\EE[F(w_T)]) + \sum_{t=1}^{T-1} \frac{D_V\delta}{(t+\mu)^2}\leq M_3,\nonumber
\end{eqnarray}
for some $M_3>0$ independent on $T$. Furthermore
\begin{eqnarray}
    \sum_{t=1}^{T-1} \frac{\gamma}{t+\mu} \EE\left[\left\|G_t\right\|_2^2\right] &\leq&\sum_{t=1}^{T-1} \frac{\gamma}{t+\mu} \EE\left[\left\|\widehat G_t - G_t\right\|_2^2\right] + 2(\EE[F(w_1)]-\EE[F(w_T)])\nonumber\\
    &\leq& \frac{\gamma}{\beta^*-C_2\gamma} M_3 +2(\EE[F(w_1)]-\EE[F(w_T)]) := M_4.\nonumber
\end{eqnarray}
We apply the lower bound
\begin{eqnarray}
    \sum_{t=1}^{T}\frac{1}{t+\mu} \geq \int_{1}^{T+1} \frac{1}{t+\mu}dt = \log\left(\frac{T+1+\mu}{1+\mu}\right),\nonumber
\end{eqnarray}
to obtain
\begin{eqnarray}
    \gamma\log\left(\frac{T+1+\mu}{1+\mu}\right) \min_{1\leq t \leq T} \EE\left[\|\widehat G_t - G_t\|_2^2\right] &\leq& \min_{1\leq t \leq T}\EE\left[\|\widehat G_t - G_t\|_2^2\right] \sum_{t=1}^T \gamma_t\nonumber\\
    &\leq& \sum_{t=1}^T \gamma_t \EE\left[\|\widehat G_t - G_t\|_2^2\right] \leq M_3,\nonumber
\end{eqnarray}
and similarly,
\begin{eqnarray}
     \gamma\log\left(\frac{T+1+\mu}{1+\mu}\right)\min_{1\leq t \leq T} \EE\left[\|\widehat G_t\|_2^2\right] &\leq& M_4.\nonumber
\end{eqnarray}
Therefore,
\begin{eqnarray}
    \min_{1\leq t \leq T} \EE\left[\|\widehat G_t - G_t\|_2^2\right] \leq \OO\left(\frac{1}{\log T}\right)\nonumber\\
    \min_{1\leq t \leq T} \EE\left[\|G_t\|_2^2\right] \leq \OO\left(\frac{1}{\log T}\right).\nonumber
\end{eqnarray}
\end{proof}

\color{black}
\section{Experimental Details}
\subsection{Values of $\lambda$}

In the experiments for strongly convex losses, the $l_2$ regularization parameters $\lambda$ are decided as described in Section~\ref{sec:convex_experiments}. In Table~\ref{tab:lambdas}, $\lambda^*$ denotes the optimal value of $\lambda$ and $\lambda_1$ to $\lambda_5$ denote the sampled values of $\lambda$ when $\lambda^*=0$.

\begin{table}[h]
\centering
\begin{tabular}{r|rrrrr}
Dataset   & SmallNorb & CovType  & News20 & KDD   & MNIST \\ \hline
$\lambda^*$ & 0.001 & $5\cdot 10^{-4}$& - & - & -\\
$\lambda_1$ &-&-    & 0.460   & 0.970  & 1.600   \\
$\lambda_2$ &-&-      & 1.100    & 0.023 & 1.400   \\
$\lambda_3$ &-&-      & 0.034  & 0.210  & 0.034 \\
$\lambda_4$ &-&-      & 0.078  & 0.088 & 0.680  \\
$\lambda_5$ &-&-      & 0.212  & 0.390  & 0.460 
\end{tabular}
\caption{$l_2$ regularization parameters}
\label{tab:lambdas}
\end{table}

\subsection{Network Architectures}

Next, we describe the FFNs used in the non-convex experiments. Each FFN has two fully connected layers with the dimensions shown in Table~\ref{tab:FFNStructure}. The underlying activation function is ReLU.

\begin{table}[H]
    \centering
    \begin{tabular}{c|c|c|c}
    \hline
      Dataset   & Input dimension & Hidden dimension & Output dimension\\\hline
        CovType & 98 & 10 & 7\\ 
        KDD & 116 & 10 & 2\\ \hline
    \end{tabular}
    \caption{Feedforward network structure}
    \label{tab:FFNStructure}
\end{table}

\subsection{Additional Results}

Additional experimental results are displayed next. Table~\ref{tab:strCvxwithMom} extends the results in Table~\ref{tab:strcvxresult} by adding the column of Momentum. Figures~\ref{fig:cvxAdd} and \ref{fig:noncvxAdd} display the training loss curves of additional convex and non-convex experiments. Specifically, the tasks are training logistic models on SmallNorb, News20 and FFNs on CovType and KDD. The y-axes show the min-max standardized training loss values. Figure~\ref{fig:cvx_errbar} shows the standard deviation of training loss over multiple runs for convex experiments. The standard deviation for KDD is not included as the difference among runs is minor. The y-axes show the min-max standardized standard deviation values.

\begin{table}[h]
\centering
\begin{tabular}{r|r|r|r|r|r}
Dataset                 & $\lambda$   & Momentum & Full TO                & Diagonal TO          & RankOne TO          \\ \hline
SmallNorb               & $\lambda^*$ & $-2.3\;(-)$      & N/A                    & $\mathbf{18.0}\;(+)$ & $13.0\;(+)$         \\ \hline
CovType                 & $\lambda^*$ & $-1.3 \;(\sim)$ & $\mathbf{3.3}\;(\sim)$ & $-2.0\;(\sim)$       & $-2.1\;(\sim)$      \\ \hline
\multirow{5}{*}{KDD}    & $\lambda_1$ & $-85.9\;(-)$      & $15.0\;(+)$            & $\mathbf{16.0}\;(+)$ & $14.0\;(+)$         \\
                        & $\lambda_2$ & $-2.6 \;(\sim)$      & $6.0\;(+)$             & $\mathbf{5.8}\;(+)$  & $-3.0\;(\sim)$      \\
                        & $\lambda_3$ & $-21.2\;(-)$      & $10.0\;(+)$            & ${10.0\;(+)}$  & $\mathbf{10.2}\;(+)$         \\
                        & $\lambda_4$ & $-11.8 \;(-)$    & $\textbf{9.0}\;(+)$    & $7.6\;(+)$           & $-6.0\;(\sim)$      \\
                        & $\lambda_5$ & $-40.2\;(-)$      & $11.0\;(+)$            & $\mathbf{11.2}\;(+)$ & $8.7\;(+)$          \\ \hline
\multirow{5}{*}{News20} & $\lambda_1$ & $-1.4 \;(-)$      & N/A                    & $\mathbf{0.4}\;(+)$  & $0.4\;(+)$          \\
                        & $\lambda_2$ & $0.2 \;(+)$     & N/A                    & $\mathbf{0.2}\;(+)$  & $0.2\;(+)$          \\
                        & $\lambda_3$ & $-1.0\;(\sim)$      & N/A                    & $\mathbf{4.4}\;(+)$  & $4.3\;(+)$          \\
                        & $\lambda_4$ & $2.2\;(+)$      & N/A                    & $\mathbf{4.8}\;(+)$  & $4.8\;(+)$          \\
                        & $\lambda_5$ & $2.0\;(\sim)$      & N/A                    & $\mathbf{16.8}\;(+)$ & $16.7\;(+)$         \\ \hline
\multirow{5}{*}{MNIST}  & $\lambda_1$ & $-9.1\;(-)$      & $\mathbf{4.2}\;(+)$    & $3.9\;(+)$           & $-1.1\;(-)$         \\
                        & $\lambda_2$ & $-12.1\;(-)$      & $\mathbf{4.4}\;(+)$    & $3.9\;(+)$           & $-1.0\;(-)$         \\
                        & $\lambda_3$ & $-13.9\;(-)$      & $-7.1\;(-)$            & $\mathbf{0.7}\;(+)$  & $\mathbf{0.7}\;(+)$ \\
                        & $\lambda_4$ & $\mathbf{12.7}\;(+)$      & ${0.3}\;(+)$    & $0.2\;(+)$           & $0.2\;(+)$          \\
                        & $\lambda_5$ & $\mathbf{3.1}\;(+)$      & ${0.7}\;(+)$    & $0.6\;(+)$           & $-3.4\;(-)$        
\end{tabular}
\caption{Values of $\rho$ for strongly convex tasks (all values $\times 10^{-3}$). Full TO denotes for full version of pseudo-linear TO in Algorithm~\ref{alg:linearTVR}. `+', `$\sim$' and `-' denote for $s< 0.05$, $0.05\leq s\leq 0.95$, and $s>0.95$, respectively.}
\label{tab:strCvxwithMom}
\end{table}

\begin{figure}[h]
    \centering
    \begin{subfigure}[b]{0.49\linewidth}
        \includegraphics[width=\linewidth]{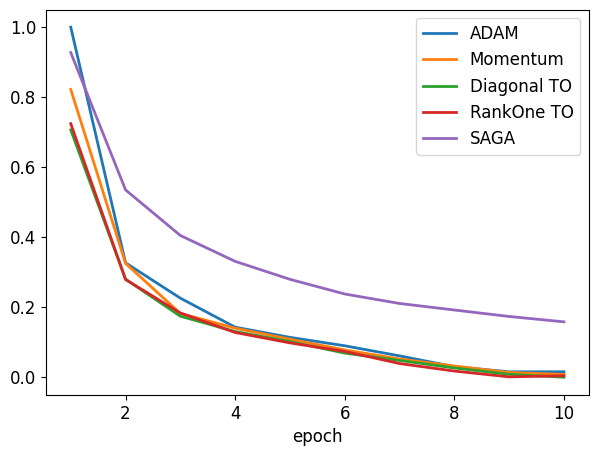}
        \caption{SmallNorb}
    \end{subfigure}
    \begin{subfigure}[b]{0.49\linewidth}
        \includegraphics[width=\linewidth]{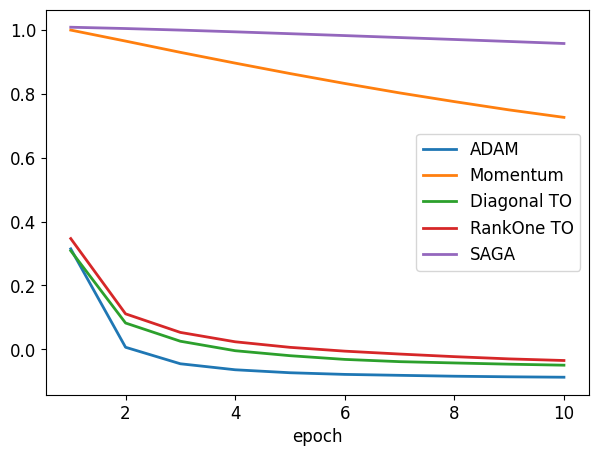}
        \caption{News20}
    \end{subfigure}
    \caption{Training loss curves for logistic regression tasks with $\lambda=0$. Standardization is based on minimum=0.96 and maximum=1.31 in the left figure, and minimum=2.18 and maximum=2.99 in the right figure.}
    \label{fig:cvxAdd}
\end{figure}

\begin{figure}
    \centering
    \begin{subfigure}[b]{0.49\linewidth}
        \includegraphics[width=\linewidth]{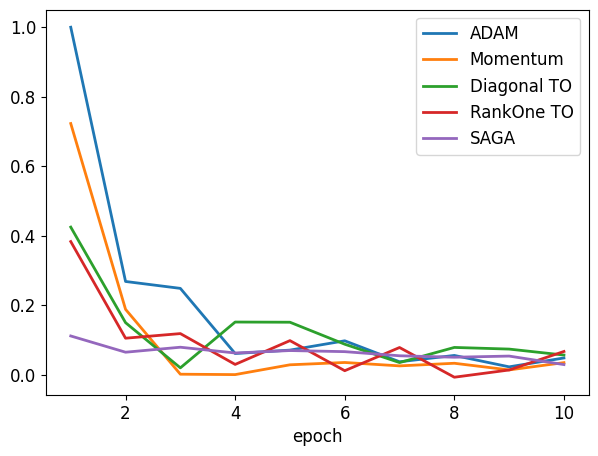}
        \caption{SmallNorb}
    \end{subfigure}
    \begin{subfigure}[b]{0.49\linewidth}
        \includegraphics[width=\linewidth]{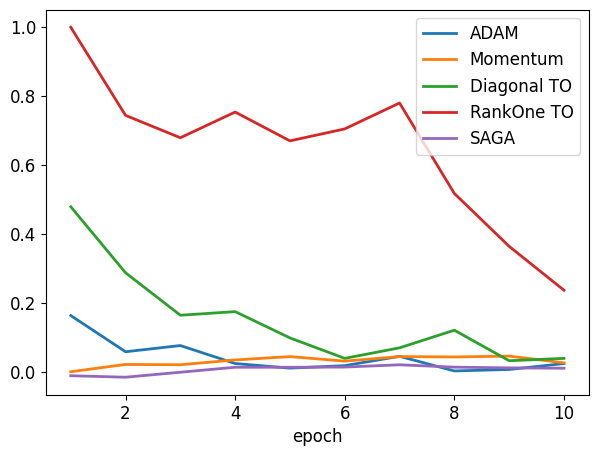}
        \caption{News20}
    \end{subfigure}
    
    \begin{subfigure}[b]{0.49\linewidth}
        \includegraphics[width=\linewidth]{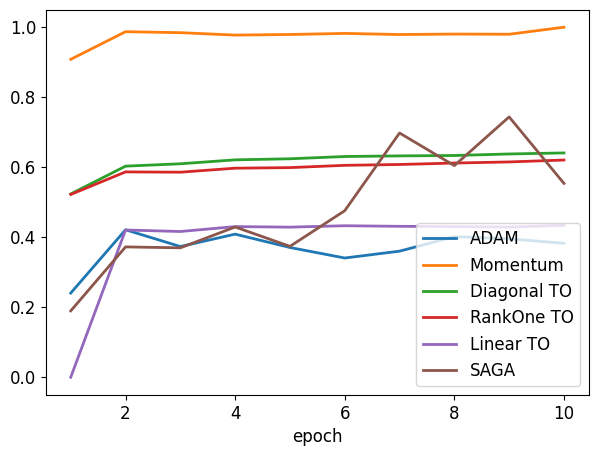}
        \caption{CovType}
    \end{subfigure}
    \begin{subfigure}[b]{0.49\linewidth}
        \includegraphics[width=\linewidth]{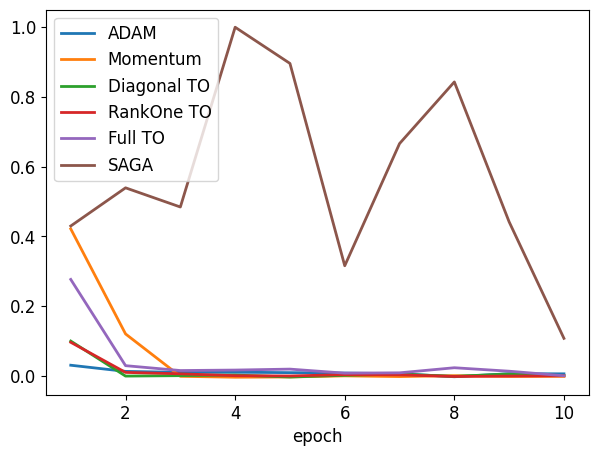}
        \caption{MNIST}
    \end{subfigure}
    \caption{Standard deviation of training loss for logistic regression tasks with $\lambda=0$. Standardization is based on minimum=$1.31\times10^{-3}$ and maximum=$4.86\times 10^{-2}$ in Figure (a), minimum=$2.78\times 10^{-4}$ and maximum=$1.36\times 10^{-2}$ in Figure (b), minimum=$1.00\times 10^{-2}$ and maximum=$2.11\times 10^{-2}$ in Figure (c), and minimum=$1.48\times 10^{-4}$ and maximum=$6.48\times 10^{-3}$ in Figure (d).}
    \label{fig:cvx_errbar}
\end{figure}

\begin{figure}
    \centering
    \begin{subfigure}[b]{0.49\linewidth}
        \includegraphics[width=\linewidth]{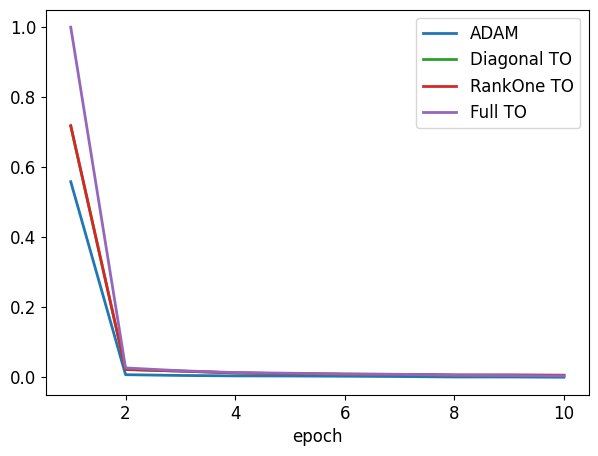}
        \caption{FFN-CovType}
    \end{subfigure}
    \begin{subfigure}[b]{0.49\linewidth}
        \includegraphics[width=\linewidth]{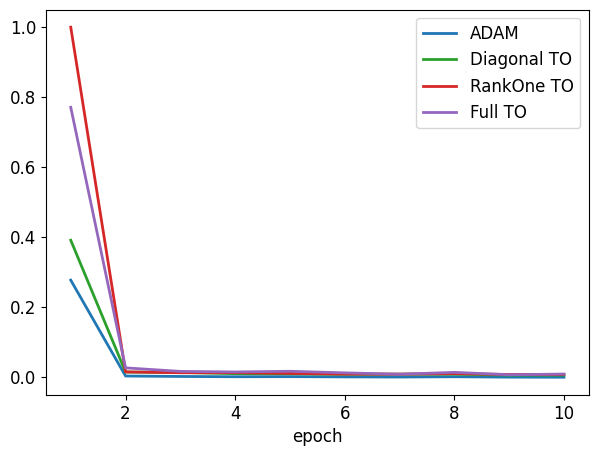}
        \caption{FFN-KDD}
    \end{subfigure}
    \caption{Training loss curves for non-convex tasks. Standardization is based on minimum=1.48 and maximum=1.53 in the left figure, and minimum=0.32 and maximum=0.35 in the right figure. }
    \label{fig:noncvxAdd}
\end{figure}

\end{document}